\documentclass{article}
\usepackage{iclr2027_conference,times}
\usepackage{hyperref}
\usepackage{comment}
\usepackage{url}
\usepackage{bm}
\usepackage{amsfonts,amsmath,amssymb,amsthm,mathtools}
\usepackage{booktabs}
\usepackage[linesnumbered,ruled,vlined,algo2e]{algorithm2e}
\SetAlFnt{\small}

\newif\ifpreprint
\preprinttrue

\ifpreprint
\iclrfinalcopy
\fi

\newcommand{\Real}{\mathbb{R}}

\newcommand{\Matrix}[1]{\bm{\mathrm{#1}}}
\newcommand{\trans}{\intercal}
\newcommand{\diag}[2][]{\mathrm{diag}#1( #2 #1)}
\newcommand{\Norm}[3][]{#1\| #2 #1\|_{\mathrm{#3}}}
\newcommand{\One}{\Matrix{1}}
\newcommand{\Zro}{\Matrix{0}}

\newcommand{\Oh}[2][]{\mathcal{O}#1( #2 #1)}
\newcommand{\uh}{\mathrm{u}}
\newcommand{\fl}[2][]{\mathsf{fl}#1( #2 #1)}
\newcommand{\fpfmt}[1]{\texttt{#1}}

\newcommand{\smax}{\varphi}
\newcommand{\smaxon}[1]{\smax_{#1}}
\newcommand{\Smax}[2][]{\smax #1( #2 #1)}
\newcommand{\Smaxon}[3][]{\smaxon{#2} #1( #3 #1)}

\newcommand{\nrmone}{\eta}
\newcommand{\Nrmone}[2][]{\nrmone #1( #2 #1)}

\newtheorem{proposition}{Proposition}

\title{LampAttention: Look-Ahead Mixed-Precision FlashAttention for Dedicated Accelerators}

\author{Stanislav Budzinskiy\thanks{Corresponding author: \texttt{stanislav.budzinskiy@univie.ac.at}}, Marian Gloser, Tolunay Yilmaz\\
Faculty of Mathematics\\
University of Vienna, Austria\\
\AND
Ying Hong Tham\\
Huawei Heisenberg Research Center, Munich, Germany\\
\And
Yuanyi Lin, Wenyi Fang, Fan Wu\\
Huawei Technologies Co. Ltd\\
\AND
Philipp Petersen\\
Faculty of Mathematics\\
University of Vienna, Austria
}

\begin{document}

\maketitle

\ifpreprint
\fancyhead{}
\renewcommand{\headrulewidth}{0pt}
\fi

\begin{abstract}
While most attention logits can be computed in low precision without degrading numerical stability, current attention kernels fail to exploit this phenomenon.
We introduce a novel hardware-algorithm co-design in the form of mixed-precision FlashAttention.
Our method accumulates key-query products and evaluates their exponentials in 8-bit formats, then adaptively identifies sensitive sub-blocks and recomputes them in 16-bit formats.
We propose the specifications for a dedicated accelerator capable of executing this pipeline efficiently.
Simulated experiments with Qwen3 and Gemma 3 show that rerouting a selective minority of sub-blocks to high precision is sufficient to recover the baseline model performance.
\end{abstract}

\section{Introduction}
Mixed-precision floating-point arithmetic is a fundamental driver of scale in modern deep learning \citep{gupta2015deep, micikevicius2018mixed, kalamkar2019study, micikevicius2022fp8}.
In practice, this efficiency relies entirely on mixed-precision matrix products: the operands are stored in low precision, whereas the output is accumulated in high precision.
As modern AI accelerators natively support these operations, they yield massive throughput gains across both training and inference.
Furthermore, the numerical reliability of mixed-precision matrix products is guaranteed by well-established theoretical round-off error bounds \citep{blanchard2020mixed, abdelfattah2025analysis}.

Beyond isolated matrix products, the overall numerical stability of a neural network depends heavily on how local round-off errors propagate through the entire architecture and interact with its layers, especially the highly nonlinear attention layers of transformers \citep{budzinskiy2025numerical}.
A standard, conservative strategy is to accumulate all matrix products and execute all intermediate computations in 32-bit arithmetic before passing the results down the computational pipeline.
However, recent theoretical advances \citep{el2026mixed, budzinskiy2026lamp} reveal that this approach is overly restrictive and introduce a dynamic alternative: a \emph{look-ahead mixed-precision} (LAMP) framework.
The core idea of LAMP is to maximize the use of low-precision arithmetic by default, using analytical error bounds to selectively trigger high-precision recomputation only where numerical stability is at risk.
These works empirically validate LAMP on feedforward activations and transformer softmax functions, showing that large fractions of the directly preceding matrix products can be accumulated in low precision without degrading end-to-end inference accuracy.

A fundamental limitation of these prior theoretical works is that they are blind to physical hardware realities.
Their LAMP algorithms fail to account for the structural constraints imposed by practical, hardware-aware kernels such as the industry-standard FlashAttention \citep{dao2022flashattention, dao2024flashattention}.

\subsection{Contributions}
We address this gap by explicitly adapting the LAMP framework to FlashAttention.
Specifically, we develop a hardware-algorithm co-design (Sections~\ref{sec:lamp_online} and~\ref{sec:lampattn}):
\begin{itemize}
    \item \emph{LampAttention}, a LAMP modification of FlashAttention that identifies numerically destabilizing logits in two stages, guided by running maxima and an analytical error bound;
    \item a specification for a hypothetical dedicated accelerator capable of executing LampAttention efficiently, featuring native 8-bit accumulation and LUT-based evaluation of exponentials.
\end{itemize}
\emph{LampAttention is distinctly not a quantization technique}.
Rather than compressing model weights or the input key-query-value operands, it exclusively minimizes the precision of intermediate calculations inside the attention kernel: the matrix accumulations and the softmax exponentials.
Therefore, the algorithm operates independently of, and seamlessly alongside, external optimizations such as operand quantization or rotary positional encodings \citep{su2024roformer}.

Targeting the massive computational demands of LLM deployment, we evaluate LampAttention for inference across two modern LLM architectures, Qwen3 \citep{yang2025qwen3} and Gemma 3 \citep{kamath2025gemma}, measuring its impact on perplexity and downstream tasks (Section~\ref{sec:experiments}).
Because commercially available hardware lacks native support for 8-bit accumulation, we simulate the arithmetic to guarantee numerical fidelity.
Consequently, rather than reporting wall-clock timings, we quantify the volume of high-precision recomputations against task metrics, mapping the efficiency-accuracy trade-off attainable by future accelerators.

\subsection{Related work}
\paragraph{Transformers.} The attention mechanism was originally introduced in \cite{vaswani2017attention} in the context of machine translation and forms the basis of all modern transformers.
These have achieved widespread success in a variety of natural language tasks, with prominent examples including BERT \citep{devlin2019bert}, GPT \citep{brown2020language}, PaLM \citep{chowdhery2023palm}, Llama \citep{touvron2023llama}, Qwen \citep{yang2025qwen3}, and DeepSeek \citep{liu2024deepseek}.
Transformers and related attention-based architectures have also enabled advances in scientific and mathematical applications, including accurate protein structure prediction with AlphaFold 2 \citep{jumper2021alphafold} and Olympiad-level geometry problem solving with AlphaGeometry \citep{trinh2024solving}.

\paragraph{Efficient attention.}
Despite their success, transformers are bound by a fundamental bottleneck: the compute and memory costs of standard attention scale quadratically with sequence length.
The family of FlashAttention algorithms \citep{dao2022flashattention, dao2024flashattention, shah2024flashattention, zadouri2026flashattention} achieves linear memory cost by reorganizing the standard formulation in an I/O-aware manner.
Sparse attention \cite{beltagy2020longformer, kitaev2020reformer} uses local sliding windows or fixed stride patterns to discard specific token interactions.
Linear \citep{katharopoulos2020transformers} and log-linear \citep{guo2026log} variants of attention remove the softmax function.
KV-caching \citep{dai2019transformer} prevents redundant recomputation of past tokens during generation, while methods such as grouped-query attention \citep{ainslie2023gqa} reduce the memory footprint of the KV-cache.

\paragraph{Quantization.}
To alleviate disk-storage and runtime memory requirements, quantization reduces the precision of model weights and activations.
Weight-only quantization \citep{frantar2023gptq, lin2024awq} compresses static model parameters.
Weight-activation quantization \citep{dettmers2022gpt3, xiao2023smoothquant} compresses both weights and activations, preparing them to become mixed-precision matrix-product operands.
\emph{LampAttention performs neither.} It exclusively controls the precision of calculations themselves.

\paragraph{Mixed-precision arithmetic.}
Frameworks for mixed-precision deep learning \citep{micikevicius2018mixed, kalamkar2019study} accelerate training and inference by allocating different numerical formats to different components of the computational pipeline.
This assignment is fundamentally inter-operational and homogeneous: both operands of a mixed-precision matrix product are quantized to one precision\footnote{Heterogeneous quantization of operands was proposed in \cite{dettmers2022gpt3}.} and its entire output is accumulated in another precision; all of the softmax exponentials are computed in the same precision.
In contrast, \emph{LampAttention introduces intra-operational, heterogeneous mixed precision}.

\section{Look-ahead mixed-precision computation of online softmax}
\label{sec:lamp_online}
The principal enabling feature of FlashAttention is the \emph{online} computation of the softmax function
\begin{equation*}
    \Smax{\Matrix{x}} = \frac{1}{\sum_{i = 1}^{n} \exp(x_i)} \begin{bmatrix}
        \exp(x_1) & \cdots & \exp(x_n)
    \end{bmatrix}^\trans, \quad \Matrix{x} \in \Real^n,
\end{equation*}
whereby an intermediate result is updated with new logits $\Matrix{y} \in \Real^m$.
Let $\smaxon{\Matrix{x}} : \Real^{m} \to \Real^{m + n}$ be the corresponding update function.
It is computed as follows.
The algorithm keeps track of
\begin{equation*}
    \mu \gets \max_{1 \leq i \leq n} x_i, \quad \omega \gets \sum_{i = 1}^{n} \exp(x_i - \mu), \quad \Matrix{s} \gets \Smax{\Matrix{x}},
\end{equation*}
and updates their running values according to
\begin{gather*}
    \hat{\mu} \gets \max \Big\{ \mu, \max_{1 \leq i \leq m} y_i \Big\}, \quad \hat{\alpha} \gets \exp(\mu - \hat{\mu}) \omega, \quad \hat{\Matrix{e}} \gets \begin{bmatrix}
        \exp(y_1 - \hat{\mu}) & \cdots & \exp(y_m - \hat{\mu})
    \end{bmatrix}^\trans, \\
    \hat{\omega} \gets \hat{\alpha} + \Norm{\hat{\Matrix{e}}}{1}, \quad \hat{\Matrix{s}} \gets \frac{1}{\hat{\omega}} \begin{bmatrix}
        \hat{\alpha} \Matrix{s}^\trans & \hat{\Matrix{e}}^\trans
    \end{bmatrix}^\trans = \Smaxon{\Matrix{x}}{\Matrix{y}}.
\end{gather*}

\subsection{Floating-point normalization}
To analyze the effects of rounding errors on the computation of $\smaxon{\Matrix{x}}$, we focus on its compositional structure and introduce the $\ell_1$ normalization function $\Nrmone{\Matrix{z}} = \Matrix{z} / \Norm{\Matrix{z}}{1}$, so that
\begin{equation*}
    \Smaxon{\Matrix{x}}{\Matrix{y}} = \Nrmone{\Matrix{z}}, \quad \Matrix{z} = \begin{bmatrix}
        \hat{\alpha} \Matrix{s}^\trans & \hat{\Matrix{e}}^\trans
    \end{bmatrix}^\trans \in \Real^{n + m}.
\end{equation*}
Then the componentwise error of floating-point (FP) evaluation of $\smaxon{\Matrix{x}}$ can be bounded by\footnote{We denote by $\fl{\cdot}$ the FP evaluation of the entire expression, e.g., $\fl{\Nrmone{\Matrix{z}}} := \fl{\Nrmone{\fl{\Matrix{z}}}}$.}
\begin{equation*}
    |\fl{\Nrmone{\Matrix{z}}} - \Nrmone{\Matrix{z}}| \leq |\fl{\Nrmone{\Matrix{z}}} - \Nrmone{\fl{\Matrix{z}}}| + |\Nrmone{\fl{\Matrix{z}}} - \Nrmone{\Matrix{z}}|,
\end{equation*}
where the first term is the rounding error of normalization at the evaluated value of $\Matrix{z}$ and the second term is the discrepancy between the exactly normalized exact and evaluated values of $\Matrix{z}$.
Denoting by $\uh_{\nrmone}$ the FP precision used to evaluate $\nrmone$, we get with standard techniques \citep{higham2002accuracy}
\begin{equation}
\label{eq:first_term_bound}
    |\fl{\Nrmone{\Matrix{z}}} - \Nrmone{\Matrix{z}}| \leq \gamma_{\nrmone, m+1} \Nrmone{\Matrix{z}} + (1 + \gamma_{\nrmone, m+1}) |\Nrmone{\fl{\Matrix{z}}} - \Nrmone{\Matrix{z}}|, \quad \gamma_{\nrmone, m+1} = \frac{(m+1) \uh_{\nrmone}}{1 - (m+1) \uh_{\nrmone}}.
\end{equation}
The bound is now determined by how the $\ell_1$ normalization function $\nrmone$ propagates the rounding errors in the computation of the exponentials and their arguments.

\subsection{Floating-point inputs of normalization}
Let us now consider the evaluation of $\Matrix{z}$, taking into account the rounding errors due to FP evaluation of logits $\Matrix{y}$.
Specifically, suppose that $\Delta y_j = |\fl{y_j} - y_j| = \Oh{\uh_{y_j}}$ for $1 \leq j \leq m$, and denote $\Delta \hat{\mu} = |\fl{\hat{\mu}} - \hat{\mu}|$.
Then standard worst-case rounding error analysis yields
\begin{gather*}
    |\fl{\hat{\alpha} s_i} - \hat{\alpha} s_i| \leq \hat{\alpha} s_i \Big[ (1 + \gamma_{\hat{\alpha}, 3}) \exp\Big( \uh_{\hat{\alpha}} (\hat{\mu} - \mu + \Delta \hat{\mu}) + \Delta \hat{\mu} \Big) - 1 \Big] = \hat{\alpha} s_i w_0, \\
    |\fl{\hat{e}_j} - \hat{e}_j| \leq \hat{e}_j \Big[ (1 + \uh_{\hat{e}_j}) \exp\Big( \uh_{\hat{e}_j} (\hat{\mu} - y_j + \Delta \hat{\mu} + \Delta y_j) + \Delta \hat{\mu} + \Delta y_j \Big) - 1 \Big] = \hat{e}_j w_j.
\end{gather*}
It follows from Taylor's expansion that
\begin{equation}
\label{eq:second_term_bound}
    |\Nrmone{\fl{\Matrix{z}}} - \Nrmone{\Matrix{z}}| \leq |\Matrix{J}_{\nrmone}(\Matrix{z}) \diag{\Matrix{z}}| \hat{\Matrix{w}} + \Oh{\Norm{\hat{\Matrix{w}}}{\infty}^2}, \quad \hat{\Matrix{w}} = \begin{bmatrix}
        w_0 \One_n^\trans & \Matrix{w}^\trans
    \end{bmatrix}^\trans \in \Real^{n + m},
\end{equation}
where $\Matrix{J}_{\nrmone}(\Matrix{z}) \in \Real^{(n+m) \times (n+m)}$ is the Jacobian and $\hat{\Matrix{w}}$ describes the relative rounding errors of $\fl{\Matrix{z}}$, where distinct precisions are used for its components.

It remains to bound $\Delta \hat{\mu}$, for which we introduce a reasonable assumption that rounding errors in the evaluation of $\Matrix{y}$ do not switch its maximizer, i.e., $y_c = \max_j y_j$ and $\fl{y_c} = \max_j \fl{y_j}$.
Then
\begin{equation*}
    \Delta \hat{\mu} \leq \begin{cases}
        0, & \mu \geq \max\{ y_c, \fl{y_c} \}, \\
        \Delta y_c, & \text{otherwise}.
    \end{cases}
\end{equation*}

We make another structural decision and assume that $\hat{\alpha}$ is computed in high precision.
This choice imposes negligible computational overhead since a single exponential is evaluated, and in the limiting case of $\uh_{\hat{\alpha}} = 0$, which we shall adopt, it yields $w_0 = \exp(\Delta \hat{\mu}) - 1$.

\subsection{Propagation of mixed-precision rounding errors}
To simplify presentation, and without loss of generality, we let $\uh_j = \uh_{y_j} = \uh_{\hat{e}_j}$ for each $1 \leq j \leq m$.
Suppose that these precisions are divided into two categories: 
\begin{equation*}
    \uh_j = \begin{cases}
        \uh, & j \not\in \Omega, \\
        \epsilon\uh, & j \in \Omega,
    \end{cases} \quad \Omega \subseteq \{ 1, \ldots, m \}, \quad \epsilon \ll 1.
\end{equation*}
That is, the exponential $\hat{e}_j$ and the logit $y_j$ are evaluated in high precision when $j \in \Omega$ and in lower precision when $j \not\in \Omega$.
Henceforth, we resort to the limiting case of $\epsilon = 0$ to streamline the analysis.

Consider a binary vector $\Matrix{g} \in \{ 0, 1 \}^{m}$ that vanishes exactly on $\Omega$.
We can represent $\hat{\Matrix{w}}$ as
\begin{equation}
\label{eq:h_via_g}
    \hat{\Matrix{w}} = \diag{\Matrix{h}} \hat{\Matrix{w}}, \quad \Matrix{h} = \begin{cases}
        \begin{bmatrix}
            \Zro_n^\trans & \Matrix{g}^\trans    
        \end{bmatrix}^\trans, & \mu \geq \max\{ y_c, \fl{y_c} \}~\text{or}~c \in \Omega, \\
        \One_{n+m}, & \text{otherwise}.
    \end{cases}
\end{equation}
The binary vector $\Matrix{h} \in \{ 0, 1 \}^{n + m}$ specifies how the rounding errors propagate through $\nrmone$.

\begin{proposition}
\label{proposition:error_propagation}
Let $\Matrix{z} \in \Real_{+}^{n + m}$, $\hat{\Matrix{w}} \in \Real^{n + m}$, and $\Matrix{h} \in \{ 0, 1 \}^{n + m}$.
For every $1 \leq l \leq n + m$,
\begin{equation*}
    |\Matrix{e}_l^\trans \Matrix{J}_{\nrmone}(\Matrix{z}) \diag{\Matrix{z}} \diag{\Matrix{h}}| \hat{\Matrix{w}} \leq \Nrmone{\Matrix{z}}_l \Big[ \Matrix{h}^\trans \Nrmone{\Matrix{z}} + h_l \Big( 1 - 2 \Nrmone{\Matrix{z}}_l \Big) \Big] \Norm{\hat{\Matrix{w}}}{\infty}.
\end{equation*}
\end{proposition}
\begin{proof}
See Appendix~\ref{appendix:proof}.
\end{proof}

Let us combine Proposition~\ref{proposition:error_propagation} with error bounds \eqref{eq:first_term_bound} and \eqref{eq:second_term_bound}.
Summing over $1 \leq l \leq n + m$, 
\begin{equation}
\label{eq:mixed_precision_error}
    \Norm{\fl{\Nrmone{\Matrix{z}}} - \Nrmone{\Matrix{z}}}{1} \leq \gamma_{\nrmone, m+1} + 2 (1 + \gamma_{\nrmone, m+1}) \Big[ \Matrix{h}^\trans \diag[\big]{\One_{n + m} - \Nrmone{\Matrix{z}}} \Nrmone{\Matrix{z}} \Big] \Norm{\hat{\Matrix{w}}}{\infty} + \Oh{\Norm{\hat{\Matrix{w}}}{\infty}^2}.
\end{equation}
This bound controls the probability mass shift caused by the FP evaluation of softmax and is maximized when $\Matrix{h} = \One_{n + m}$, i.e., when the running maximum $\hat{\mu}$ is set to a low-precision logit.

\subsection{Look-ahead distribution of precisions}
Our error bound \eqref{eq:mixed_precision_error} represents a trade-off between cheap, low-precision evaluation of softmax and its numerical stability: maximum efficiency is achieved with $\Matrix{g} = \One_m$, whereas $\Matrix{g} = \Zro_m$ guarantees unconditional numerical stability.
This setting naturally lends itself to LAMP analysis \citep{el2026mixed, budzinskiy2026lamp}.
Specifically, we can optimize
\begin{equation}
\label{eq:lamp}
    \One_m^\trans \Matrix{g} \to \max \quad \text{s.t.} \quad \Matrix{h}^\trans \diag[\big]{\One_{n + m} - \Nrmone{\Matrix{z}}} \Nrmone{\Matrix{z}} \leq \tau,
\end{equation}
where $0 \leq \tau < 1$ is a small threshold and $\Matrix{h}$ is derived from $\Matrix{g}$ via \eqref{eq:h_via_g}.
However, we do not have access to the exact values of $\Nrmone{\Matrix{z}}$ in practice.
Additionally, the explicit storage of the first $n$ softmax probabilities would render any method inefficient in long-context scenarios.

A two-step procedure can overcome these issues.
First, we look at the logits $\fl{\Matrix{y}}$ evaluated in low precision and identify their maximizer $\fl{y_c}$.
If $\mu > \fl{y_c} + |\mu|\delta$ with a safety margin $\delta \geq 0$ then we proceed; otherwise, $\fl{y_c}$ is recomputed in high precision.
In either case, we end up in the first branch of \eqref{eq:h_via_g},\footnote{Technically, this would also require $|\mu|\delta$ to be greater than the largest $\Delta y_j$. While it is possible to use worst-case upper bounds on $\Delta y_j$, they tend to be overly pessimistic, so we prefer to use a tunable margin $\delta$.} and therefore the optimization problem \eqref{eq:lamp} relies entirely on the newly added softmax probabilities.
Denote by $\Matrix{p} \in \Real^m$ these last $m$ components of $\fl{\Nrmone{\Matrix{z}}}$ evaluated according to the precisions chosen for $\fl{\Matrix{y}}$.
Then we can reformulate \eqref{eq:lamp} equivalently as
\begin{equation}
\label{eq:lamp_final}
    \One_m^\trans \Matrix{g} \to \max \quad \text{s.t.} \quad \Matrix{g}^\trans \diag[\big]{\One_{m} - \Matrix{p}} \Matrix{p} \leq \tau, \quad g_c = 0~~\text{if}~~\mu \leq \fl{y_c} + |\mu| \delta.
\end{equation}
This is a knapsack problem with uniform values, and hence its optimal solution can be obtained with a greedy algorithm: select the smallest entries of $\diag[\big]{\One_{m} - \Matrix{p}} \Matrix{p}$ until the threshold is surpassed, subject to the constraint on $g_c$.
The solution $\Matrix{g}$ determines the distribution of precisions across logits, and we reevaluate them, together with their exponentials, in higher precision accordingly.
Finally, we recompute the $\ell_1$ normalization for its updated inputs.

By restricting the binary vector $\Matrix{h}$ in \eqref{eq:h_via_g} strictly to the first branch, we ensure that the LAMP evaluation of online softmax can be seamlessly integrated into FlashAttention.
Indeed, the method relies only on the running maximum and the running normalization constant, completely avoiding the need to materialize the historical probability distribution.

\section{Look-ahead mixed-precision FlashAttention}
\label{sec:lampattn}
The framework developed in Section~\ref{sec:lamp_online} deviates from the practice of FlashAttention in three aspects, two of which stem from the block-based nature of efficient matrix multipliers in modern AI accelerators.
First, instead of looking at a single vector of logits $\Matrix{y} \in \Real^m$ at a time, FlashAttention processes an entire block of queries simultaneously, leading to a matrix of logits $\Matrix{Y} \in \Real^{m \times b_q}$.
Second, assuming that the ``atomic'' matrix product is $(b_{k} \times b) \times (b \times b_{q}) \to (b_{k} \times b_{q})$, the number of new keys per query is $m = b_k N_k$ and the (re)computation of logits is strictly blockwise.
Third, the output of attention is scaled by the normalization constant only at the very end of the computational pipeline, i.e., FlashAttention keeps the intermediate exponentials unnormalized.

\subsection{Block modification of LAMP computation of online softmax}
Let $\Matrix{\mu} \in \Real^{b_q}$ be the vector of running maxima for each query and denote by $\Matrix{C} \in \{ 1, \ldots, b_k \}^{N_k \times b_q}$ the matrix of maximum-logit indices for each sub-block of low-precision $\fl{\Matrix{Y}}$.
Consider
\begin{equation*}
    T_{\theta} = \max_{1 \leq j \leq b_q} \big\{ \fl{Y_{i,j}} - \mu_{j} \big\}, \quad i = C_{\theta, j} + (\theta - 1) b_k, \quad 1 \leq \theta \leq N_k.
\end{equation*}
This quantity represents the ``threat'' caused by the $\theta$th sub-block, and we assume that the sub-blocks are sorted so that $T_1 \geq \cdots \geq T_{N_k}$.
Let $\hat{\Matrix{\mu}} = \Matrix{\mu}$ and process the sub-blocks sequentially.
If
\begin{equation}
\label{eq:safety_margin}
    \hat{\mu}_j > \fl{Y_{i,j}} + |\hat{\mu}_j| \delta, \quad i = C_{\theta, j} + (\theta - 1) b_k \quad \text{for all} \quad 1 \leq j \leq b_q,
\end{equation}
the $\theta$th sub-block of $\fl{\Matrix{Y}}$ remains in low precision; otherwise, it is recomputed in high precision and the running maxima $\hat{\Matrix{\mu}}$ are updated.
Denote by $\Theta \subseteq \{ 1, \ldots, N_k \}$ the set of indices of recomputed sub-blocks.
The role of sorting according to $T_{\theta}$ is to attempt to reduce the cardinality of $\Theta$.

Next, we need to modify the LAMP problem \eqref{eq:lamp_final}.
For every query $1 \leq j \leq b_q$, denote by $\hat{\Matrix{e}}_j \in \Real^{b_k N_k}$ the vector of unnormalized shifted exponentials and by $\hat{\omega}_j \in \Real$ the updated running normalization constant.
Let $\Matrix{g} \in \{ 0, 1 \}^{N_k}$ encode the precisions used for the sub-blocks, so that $g_\theta = 0$ whenever the $\theta$th sub-block requires high precision.
Then we propose to solve
\begin{gather}
\label{eq:lamp_block}
    \One_{N_k}^\trans \Matrix{g} \to \max \quad \text{s.t.} \quad \xi_j(\Matrix{g}) \leq \tau \hat{\omega}_j^2~~\text{for all}~~1 \leq j \leq b_q, \quad g_\theta = 0~~\text{for}~~\theta \in \Theta, \\
    \xi_j(\Matrix{g}) = (\Matrix{g} \otimes \One_{b_k})^\trans \diag[\big]{\hat{\omega}_j \One_{b_k N_k} - \hat{\Matrix{e}}_j} \hat{\Matrix{e}}_j.\nonumber
\end{gather}
Here, $\otimes$ is the Kronecker product.
Unlike \eqref{eq:lamp_final}, this modified optimization problem is a multidimensional knapsack problem \citep{kellerer2004knapsack} and, in general, does not admit a greedy solution.
To solve it, we can rely on the number of sub-blocks $N_k$ being small and compare all feasible configurations of $\Matrix{g}$ in the descending order of $\One_{N_k}^\trans \Matrix{g}$.
Upon finding an optimal precision distribution, we recompute the identified sub-blocks, their exponentials, and the running normalization constant.

\subsection{LampAttention}
We integrate the LAMP framework into the FlashAttention-2 kernel \citep{dao2024flashattention}, omitting the warp-group instructions and specialization of FlashAttention-3 \citep{shah2024flashattention} and the asynchronous computations with dedicated memory levels of FlashAttention-4 \citep{zadouri2026flashattention}.
We assume a standard accelerator architecture characterized by (i)~a memory hierarchy that distinguishes between HBM, SRAM, and registers; (ii)~a SIMT execution model partitioned into thread blocks and further subdivided into warps.
The pseudocode of our \emph{LampAttention} algorithm is listed in Algorithm~\ref{alg:lampattn}.

The mixed-precision sections of Algorithm~\ref{alg:lampattn} are on lines 10--20 and 22--26, while the rest is vanilla FlashAttention-2.
Note that (causal) masking and RoPE \citep{su2024roformer} are seamlessly supported.

\begin{algorithm2e}[ht]
\caption{LampAttention}
\label{alg:lampattn}
\DontPrintSemicolon

\KwIn{$\Matrix{Q}_{full}, \Matrix{K}_{full}, \Matrix{V}_{full} \in \Real^{d \times L}$ stored in HBM for a single attention head and a single sequence of $L$ tokens;
query block size $b_q$ and key/value block size $b_k$ such that $L = b_q N_q M_q = b_k N_k M_k$; 
safety margin $\delta > 0$ of condition \eqref{eq:safety_margin};
threshold $0 \leq \tau < 1$ of block-LAMP problem \eqref{eq:lamp_block}}

\KwOut{attention $\Matrix{O}_{full} \in \Real^{d \times L}$}

\ForEach{$\upsilon_q = 1, \ldots, M_{q}$ in parallel}{
    \tcc{single thread block with its SRAM section}
    
    load $\upsilon_q$th tile $[\Matrix{Q}_1~\cdots~\Matrix{Q}_{N_q}] \in \Real^{d \times b_q N_q}$ into SRAM\;
    
    initialize output $\Matrix{O} = [\Matrix{O}_1~\cdots~\Matrix{O}_{N_q}] \gets \Zro_{d \times b_q N_q}$\;

    initialize running quantities $\Matrix{\mu} = [\Matrix{\mu}_1^\trans~\cdots~\Matrix{\mu}_{N_q}^\trans]^\trans \gets (-\infty) \One_{b_q N_q}$ and $\Matrix{\omega} = [\Matrix{\omega}_1^\trans~\cdots~\Matrix{\omega}_{N_q}^\trans]^\trans \gets \Zro_{b_q N_q}$\;
    
    \For{$\upsilon_k = 1, \ldots, M_{k}$}{
        load $\upsilon_{k}$th tiles $\Matrix{K} = [\Matrix{K}_1~\cdots~\Matrix{K}_{N_k}], \Matrix{V} = [\Matrix{V}_1~\cdots~\Matrix{V}_{N_k}] \in \Real^{d \times b_k N_k}$ into SRAM\;

        initialize $\hat{\Matrix{\mu}} \gets \Matrix{\mu}$\;

        \ForEach{$w = 1, \ldots, N_q$ in parallel}{
            \tcc{single warp with its SRAM subsection and registers}

            accumulate $\Matrix{Y} \gets \tfrac{1}{\sqrt{d}} \Matrix{K}^\trans \Matrix{Q}_w \in \Real^{b_k N_k \times b_q}$ in low precision\;

            find maximum-logit indices $\Matrix{C} \in \{ 1, \ldots, b_k \}^{N_k \times b_q}$ for $\Matrix{Y}$ and compute $\Matrix{T} \in \Real^{N_k}$\;

            sort $(\theta_{1}, \ldots, \theta_{N_k})$ so that $T_{\theta_{1}} \geq \cdots \geq T_{\theta_{N_k}}$ and initialize $\Theta \gets \varnothing$\;

            \For{$\theta = \theta_{1}, \ldots, \theta_{N_k}$}{
                \If{safety-margin condition \eqref{eq:safety_margin} based on $\hat{\Matrix{\mu}}_w$ fails for $\theta$th sub-block of $\Matrix{Y}$}{
                    accumulate $\Matrix{Y}_\theta \gets \tfrac{1}{\sqrt{d}} \Matrix{K}_\theta^\trans \Matrix{Q}_w \in \Real^{b_k \times b_q}$ in high precision and update $\Theta \gets \Theta \cup \{ \theta \}$\;

                    update $\hat{\Matrix{\mu}}_w \gets \max \{ \hat{\Matrix{\mu}}_w, \mathrm{colmax}(\Matrix{Y}_\theta) \}$\;   
                }
            }

            \For{$\theta = 1, \ldots, N_k$}{
                \uIf{$\theta \in \Theta$}{
                    compute in-place $\hat{\Matrix{E}}_\theta \gets \exp(\Matrix{Y}_\theta - \One_{b_k} \hat{\Matrix{\mu}}_w^\trans)$ in high precision\;
                }\Else{
                    compute in-place $\hat{\Matrix{E}}_\theta \gets \exp(\Matrix{Y}_\theta - \One_{b_k} \hat{\Matrix{\mu}}_w^\trans)$ in low precision\;
                }
            }
            
            compute $\hat{\Matrix{\lambda}} \gets \exp(\Matrix{\mu}_w - \hat{\Matrix{\mu}}_w)$, and $\hat{\Matrix{\alpha}} \gets \diag{\hat{\Matrix{\lambda}}} \Matrix{\omega}_w$, and $\hat{\Matrix{\omega}} \gets \hat{\Matrix{\alpha}} + [\Norm{\hat{\Matrix{e}}_1}{1}~\cdots~\Norm{\hat{\Matrix{e}}_{b_q}}{1}]^\trans$\;

            obtain solution $\Matrix{g} \in \{ 0, 1 \}^{N_k}$ of block-LAMP problem \eqref{eq:lamp_block}\;

            \For{$\theta = 1, \ldots, N_k$}{
                \If{$g_\theta = 0$ and $\theta \not\in \Theta$}{
                    accumulate $\Matrix{Y}_\theta \gets \tfrac{1}{\sqrt{d}} \Matrix{K}_\theta^\trans \Matrix{Q}_w \in \Real^{b_k \times b_q}$ in high precision\;

                    compute in-place $\hat{\Matrix{E}}_\theta \gets \exp(\Matrix{Y}_\theta - \One_{b_k} \hat{\Matrix{\mu}}_w^\trans)$ in high precision\;
                }
            }

            update $\Matrix{\mu}_w \gets \hat{\Matrix{\mu}}_w$ and $\Matrix{\omega}_w \gets \hat{\Matrix{\alpha}} + [\Norm{\hat{\Matrix{e}}_1}{1}~\cdots~\Norm{\hat{\Matrix{e}}_{b_q}}{1}]^\trans$\;

            update $\Matrix{O}_w \gets \Matrix{O}_w \diag{\hat{\Matrix{\lambda}}} + \Matrix{V} \hat{\Matrix{E}}$\;
        }
    }

    scale $\Matrix{O} \gets \Matrix{O} \diag{\Matrix{\omega}}^{-1}$ and write into $\upsilon_q$th submatrix of $\Matrix{O}_{full}$ in HBM\;
}
\end{algorithm2e}

\subsection{Hardware-algorithm co-design}
In this section, we map the pseudocode of Algorithm~\ref{alg:lampattn} to the hardware of a hypothetical dedicated AI accelerator that would be able to execute it efficiently.

\paragraph{Floating-point formats.}
We propose the 8-bit \fpfmt{e4m3} format for the low-precision accumulation of key-query matrix products.
However, this choice poses a strict upper limit of 448 on the maximum allowed absolute value of a pre-softmax logit to avoid overflow.
To guarantee this bound, we require that a transformer employ headwise QK-Norm with moderate gain weights, e.g., as in Qwen3 \citep{yang2025qwen3} and Gemma 3 \citep{kamath2025gemma} models. 
For a typical head dimension of $d = 128$, the maximum attainable value of pre-softmax logit is $8\sqrt{2} \times \mathrm{maxgain}^2$.
Thus, as long as the $1/\sqrt{d}$ scaling is applied before the matrix product is accumulated, the condition $\mathrm{maxgain} \leq 6.29$ prevents overflow during \fpfmt{e4m3} accumulation.
To execute LampAttention efficiently, a dedicated accelerator needs to support native packed \fpfmt{e4m3} accumulation, whereby 4 accumulators occupy a single 32-bit register.
The swamping due to \fpfmt{e4m3} accumulation is typically benign for Algorithm~\ref{alg:lampattn}: it primarily affects ``unimportant'' sub-blocks where the resulting exponential probabilities are vanishingly small, thus contributing no meaningful error to the final attention output.
To mitigate its potential adversary effects, stochastic rounding \citep{croci2022stochastic} could be used instead of round-to-nearest.

Next, the logits are shifted by the running maximum, guaranteeing that the shifted logits are strictly non-positive.
To account for the predetermined sign and the doubled absolute-value limit of 896, we propose to compute the shifted logits in-place using an unsigned \fpfmt{ue5m3} format.

The exponentials of the shifted logits cannot exceed one.
Therefore, we can use the \fpfmt{ue5m3} format again---but with an adapted dynamic range---to compute the exponentials in-place.
Namely, we can choose the exponent bias to map the bit pattern \fpfmt{11111000} to $1.0$, which yields the smallest normalized number $2^{-30}$ and the smallest subnormal number $2^{-33}$.
Crucially, as both the exponentials and their operands are restricted to 8 bits, the evaluation can be physically implemented as a look-up table (LUT) and executed in one clock cycle, bypassing the multi-cycle special-function unit (SFU).

To carry out high-precision recomputations, we recall that the upper bound on the pre-softmax logits remains fixed.
Therefore, we propose using a 16-bit \fpfmt{e4m11} format to accumulate matrix products, maximizing mantissa width, and an unsigned \fpfmt{ue5m11} format to evaluate the shifted exponentials.
In this case, the accelerator must support native packed \fpfmt{e4m11} accumulation with two accumulators per register and must route the evaluation of exponentials through the SFU, since a 16-bit LUT would occupy a prohibitive amount of silicon space.
In addition, while the running maxima of LampAttention require \fpfmt{e4m11} storage, they can be seamlessly rounded to \fpfmt{e4m3} for the low-precision shifts by truncating the lower mantissa bits.

Finally, to compute the attention output, the accelerator must support native matrix products where one operand is stored in the \fpfmt{ue5m3} or \fpfmt{ue5m11} format with the modified exponent bias.

\paragraph{Register layout.}
To execute the inner loop over $w$ in Algorithm~\ref{alg:lampattn}, each warp must have a sufficient number of registers for both the initial 8-bit results and their 16-bit refinements.
A ``safe'' layout, which aligns perfectly with the pseudocode, is to reserve enough registers to store all $N_k$ sub-blocks of $\Matrix{Y}$ in high precision.
While this layout accommodates the worst-case scenario, it allows only two 8-bit accumulators to reside in a register, leaving space for potential recomputations, and therefore half of the register space is wasted in the best-case scenario without any recomputations.

In practice, softmax probability distributions tend to be highly concentrated for pre-trained models, and we thus expect the recomputations to be sparse.
For a ``compact'' layout, we propose to separate 8-bit and 16-bit storage in the register space.
Namely, sufficient registers are reserved for the entire low-precision $\Matrix{Y}$ with four 8-bit accumulators per register without padding.
Furthermore, we require additional register space to store a single high-precision sub-block of $\Matrix{Y}$.
In those instances where LampAttention marks two or more sub-blocks for recomputation, the algorithm should branch to a sequential, fully 16-bit update of $\Matrix{O}_w$, reusing the registers for $\lfloor N_k/2 \rfloor + 1$ high-precision sub-blocks.

\paragraph{Sources of speedup.}
The hardware-algorithm co-design of LampAttention serves to maximize the speedup relative to vanilla FlashAttention-2.
Independent of the register layout, the primary latency reduction stems from the exponentials: evaluating an 8-bit exponential in one clock cycle via a LUT is drastically more efficient than routing a high-precision value through a multi-cycle SFU.\footnote{The evaluation of exponentials is one of the two main bottlenecks addressed in FlashAttention-4 \citep{zadouri2026flashattention}, which bypasses the SFU by proposing a low-degree polynomial approximation mapped to ALUs. In contrast, we resolve this natively without occupying either the SFU or ALU.}
This operation-level efficiency is present in both the ``safe'' and ``compact'' layouts.
At the same time, the identification of salient sub-blocks and their recomputation introduces additional latency; while we purposefully omit warp-group specialization in the present paper, it could be used to mask this latency overhead \citep{shah2024flashattention}.

In contrast, the ``compact'' layout improves computational throughput by reducing register pressure.
Packing four 8-bit accumulators into a single 32-bit register reduces the accumulator footprint by 75\%.
This enables a two-step hardware optimization: first, tile dimensions can be scaled up to the physical SRAM limit to minimize HBM reads; second, the remaining register savings increase warp occupancy to better hide memory-transfer latency.
Provided that ``disruptive'' recomputations of two or more sub-blocks remain infrequent, this dual scaling can potentially yield consistent speedups.

\section{Numerical experiments}
\label{sec:experiments}
\subsection{Setup and implementation}
We implement LampAttention as a Triton kernel with simulated 8-bit and 16-bit arithmetic.
Namely, the corresponding numbers are stored in FP32 with mantissas truncated to 3 and 11 bits, respectively.
The low-precision accumulation of ``atomic'' matrix products is carried out via recursive summation with FMA, i.e., based on $\mathrm{round}(c + a \cdot b)$, where the scalar multiplication and addition are in FP32.
Low-precision exponentials are modeled as correctly rounded, $\mathrm{round}(\exp a)$.
The tile sizes are fixed as $b_q N_q = 128$ and $b_k N_k = 64$ with $N_k = 4$ key sub-blocks.

To validate the performance of LampAttention, we evaluate it on models from the Qwen3 (8B, 30B-MoE, 32B) and Gemma 3 (12B, 27B) families.
The LampAttention kernel is dynamically injected into the standard HuggingFace implementations at runtime.
We compare this computation against two fixed-precision baselines: the vanilla 32-bit baseline and an 8-bit LampAttention baseline restricted from utilizing any 16-bit recomputations. 
We stress that these bit-widths refer strictly to the attention-logit accumulators and their exponentials, rather than the quantization of model weights.
Task metrics are measured using the EleutherAI lm-evaluation-harness.

\ifpreprint
The code used for the experiments is publicly available.\footnote{https://github.com/sbudzinskiy/lamp-attention}
\else
The code and the raw numerical results (including those not reported in the paper) are provided in the Supplementary Material.
\fi
The experiments were performed on a cloud-based GPU equipped with 80 GB of VRAM and required approximately 540 GPU-hours.

\subsection{Numerical results}
Tables~\ref{tab:c4} and~\ref{tab:mmlu} show that the adaptive 16-bit recomputations of LampAttention substantially improve upon the pure 8-bit baseline across both perplexity (C4) and downstream (MMLU) tasks; two more tasks are evaluated in Appendix~\ref{appendix:wiki-arc}.
Even with a high threshold of $\tau = 0.5$ in \eqref{eq:lamp_block}, i.e., with a small number of recomputations during the second stage of Algorithm~\ref{alg:lampattn}, the recovered metrics broadly approach the 32-bit vanilla baselines (although the degree of closeness remains model- and task-dependent).
Crucially, the recovery is achieved with only about 20-30\% of key tiles being ``disruptive'' for the ``compact'' layout, requiring the recomputation of two or more sub-blocks.

\begin{table}[tbhp]
\centering
\caption{LampAttention with $\delta = 2^{-8}$ and $\tau = 2^{-1}$ on C4.}
\label{tab:c4}
\begin{tabular}{lcccccc}
\toprule
& \multicolumn{3}{c}{\textbf{Perplexity} $\bm{(\downarrow)}$} & \multicolumn{3}{c}{\textbf{16-bit sub-blocks per tile}} \\
\cmidrule(lr){2-4}
\cmidrule(lr){5-7}
\textbf{Model name} & 8-bit & \textbf{8/16-bit LAMP} & 32-bit & 0 & 1 & 2+ \\
\midrule
gemma-3-27b-pt & 21.470 & \textbf{17.150} & 16.540 & 46.30\% & 20.52\% & 33.18\% \\
gemma-3-12b-pt & 36.740 & \textbf{21.830} & 18.680 & 50.94\% & 20.30\% & 28.76\% \\
\midrule
Qwen3-32B & 28.180 & \textbf{23.740} & 23.800 & 65.09\% & 15.43\% & 19.48\% \\
Qwen3-8B & 44.030 & \textbf{33.780} & 33.340 & 66.19\% & 14.58\% & 19.23\% \\
\midrule
Qwen3-30B-A3B & 43.900 & \textbf{28.270} & 27.510 & 69.75\% & 13.71\% & 16.54\% \\
\bottomrule
\end{tabular}
\end{table}

\begin{table}[tbhp]
\centering
\caption{LampAttention with $\delta = 2^{-8}$ and $\tau = 2^{-1}$ on 5-shot MMLU.}
\label{tab:mmlu}
\begin{tabular}{lcccccc}
\toprule
& \multicolumn{3}{c}{\textbf{Accuracy} $\bm{(\uparrow)}$} & \multicolumn{3}{c}{\textbf{16-bit sub-blocks per tile}} \\
\cmidrule(lr){2-4}
\cmidrule(lr){5-7}
\textbf{Model name} & 8-bit & \textbf{8/16-bit LAMP} & 32-bit & 0 & 1 & 2+ \\
\midrule
gemma-3-27b-pt & 0.6261 & \textbf{0.7851} & 0.7993 & 48.49\% & 21.55\% & 29.96\% \\
gemma-3-12b-pt & 0.3618 & \textbf{0.5899} & 0.7654 & 53.39\% & 22.19\% & 24.42\% \\
\midrule
Qwen3-32B & 0.7555 & \textbf{0.8355} & 0.8344 & 54.52\% & 19.67\% & 25.81\% \\
Qwen3-8B & 0.5625 & \textbf{0.7412} & 0.7840 & 54.05\% & 19.11\% & 26.84\% \\
\midrule
Qwen3-30B-A3B & 0.5998 & \textbf{0.7884} & 0.8169 & 56.97\% & 18.99\% & 24.04\% \\
\bottomrule
\end{tabular}
\end{table}

Furthermore, Figure~\ref{fig:pareto_acc_2x2} depicts the empirical efficiency-accuracy trade-off of LampAttention across varying thresholds.
As $\tau$ decreases, task performance converges toward the 32-bit vanilla baseline, though non-monotonically for smaller values of $\tau$.
This observation indicates that enforcing higher precision on intermediate computations does not unconditionally guarantee superior metrics; in fact, Table~\ref{tab:arc0} in Appendix~\ref{appendix:wiki-arc} illustrates an instance where the 8-bit baseline accuracy exceeds the 32-bit baseline.
In contrast, the rate of ``disruptive'' recomputations increases monotonically as $\tau$ tends to zero, reflecting the increasingly stringent constraints of the block-LAMP optimization problem \eqref{eq:lamp_block}.

Figure~\ref{fig:pareto_acc_2x2} also provides a new perspective on the results in Tables~\ref{tab:c4}-~\ref{tab:mmlu}.
We observe that the empirical difference between $\tau = 1$ and $\tau = 0.5$ is marginal in terms of both recomputation rate and accuracy.
Because the optimal solution to \eqref{eq:lamp_block} at $\tau = 1$ strictly omits stage-two recomputations, we can deduce the structural division of labor within the algorithm:
\begin{itemize}
    \item the first, running-maximum stage of LampAttention is responsible for the ``heavy lifting'' of the evaluation metrics from the 8-bit baseline toward the 32-bit baseline using a moderate budget of ``disruptive'' recomputations for the compact register layout;
    \item the second stage, solving \eqref{eq:lamp_block}, acts as a fine-grained correction mechanism that closes the final metric gap at the cost of additional recomputations.
\end{itemize}

Finally, Figure~\ref{fig:bars_1x2} visualizes the dynamics of the recomputation cascade as the threshold $\tau$ tightens: the mass of zero-recomputation cases systematically transfers into full-recomputation cases.

Additional experimental results are presented in the Appendix.

Our findings highlight the potential utility of incorporating the new mixed-precision paradigm, the LAMP, into FlashAttention.
They expose a controllable trade-off between inference efficiency and downstream performance without resorting to the 32-bit arithmetic for the computationally intensive sections of the algorithm.

\begin{figure}[tbhp]
    \centering
    \includegraphics[width=\linewidth]{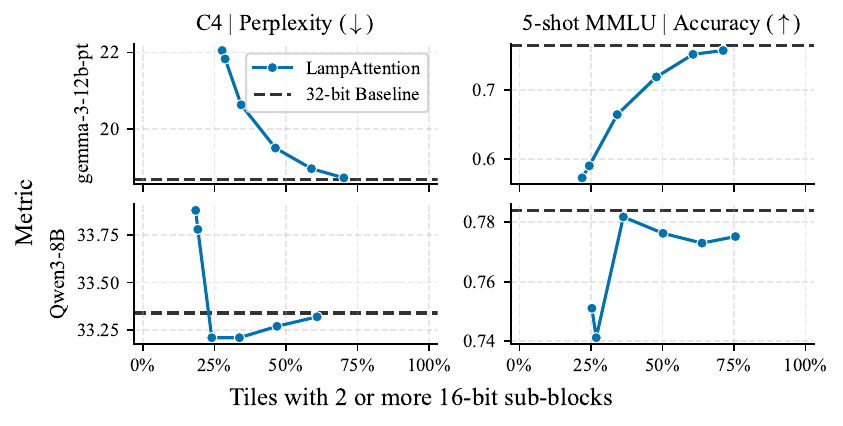}
    \caption{LampAttention with $\delta = 2^{-8}$ and $\tau \in \{ 2^{-t}~:~t = 0, \ldots, 5 \}$.}
    \label{fig:pareto_acc_2x2}
\end{figure}

\begin{figure}[tbhp]
    \centering
    \includegraphics[width=\linewidth]{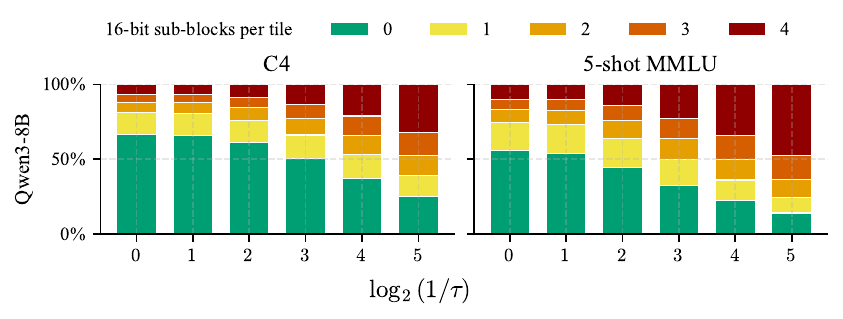}
    \caption{LampAttention with $\delta = 2^{-8}$ and $\tau \in \{ 2^{-t}~:~t = 0, \ldots, 5 \}$.}
    \label{fig:bars_1x2}
\end{figure}

\ifpreprint
\else
\clearpage
\fi

\ifpreprint
\section*{Author Contributions}
SB conceived the approach, formulated the research problem, and carried out the formal analysis, experimentation, and implementation. SB wrote the original draft of the manuscript.
MG and TY contributed to the analysis and implementation and assisted with proofreading. YHT contributed to project administration and to reviewing and editing the manuscript.
YL, WF, and FW contributed to project administration.
PP supervised the project as laboratory head and reviewed the manuscript.

\section*{Acknowledgments}
This work was carried out in the framework of a research project funded by Huawei Technologies Ltd.
\else
\section*{Ethics statement}
This paper presents work whose goal is to advance the field of machine learning by improving the efficiency of large language models. Our method has the potential to reduce the energy consumption required to deploy these models, contributing to “Green AI” initiatives.

\section*{Reproducibility statement}
To ensure full reproducibility of the results, we have included in the Supplementary Material the complete source code, the configurations of the Python environment, the exact evaluation script, and hyperparameters used during the experiments.
\fi

\section*{AI use statement}
In this work, we used generative AI tools to assist with software development and to refine the clarity and flow of the text. All core research ideation, theoretical development, experimental design, and data analysis were conducted entirely by the authors. We have reviewed all AI-assisted work and assume full responsibility for the accuracy, integrity, and originality of this work.

\bibliography{main}
\bibliographystyle{iclr2027_conference}

\clearpage
\appendix
\section{Proof of Proposition~\ref{proposition:error_propagation}}
\label{appendix:proof}
The $\ell_1$ normalization function $\eta$ is differentiable when every entry of $\Matrix{z}$ is nonzero; in our case, they are positive.
With $\Matrix{\nrmone} = \Nrmone{\Matrix{z}}$, direct calculation yields
\begin{equation*}
    \Matrix{A} = \Matrix{J}_{\nrmone}(\Matrix{z}) \diag{\Matrix{z}} \diag{\Matrix{h}} = (\Matrix{I} - \Matrix{\nrmone} \One^\trans) \diag{\Matrix{\nrmone}} \diag{\Matrix{h}}.
\end{equation*}
Since $\nrmone_i \in (0,1)$ for each $i$, the diagonal entries of $\Matrix{A}$ are nonnegative, and hence
\begin{equation*}
    |\Matrix{e}_l^\trans \Matrix{A}| \hat{\Matrix{w}} \leq \Matrix{e}_l^\trans \Matrix{A} \One \Norm{\hat{\Matrix{w}}}{\infty} =  \Big[ \nrmone_l \sum_{i \neq l} \nrmone_i h_i + (1 - \nrmone_l) \nrmone_l h_l \Big] \Norm{\hat{\Matrix{w}}}{\infty}
    = \nrmone_l \Big[ \Matrix{h}^\trans \Matrix{\nrmone} + h_l (1 - 2 \nrmone_l) \Big] \Norm{\hat{\Matrix{w}}}{\infty}.
\end{equation*}

\section{Additional benchmarks}
\label{appendix:wiki-arc}
In addition to the C4 and MMLU tasks presented in the main text, we have evaluated the performance of LampAttention on Wikitext and ARC-Challenge; see Tables~\ref{tab:wiki} and~\ref{tab:arc0}.
Notably, the 8-bit accuracy baseline of Qwen3-32B on ARC-Challenge is higher than the corresponding 32-bit baseline, indicating that there is no strict causal relationship between higher-precision intermediate calculations and higher final accuracy.
Figure~\ref{fig:bars_5x4} is an extension of Figure~\ref{fig:bars_1x2} that contains all models and tasks.

\begin{table}[htbp]
\centering
\caption{LampAttention with $\delta = 2^{-8}$ and $\tau = 2^{-1}$ on Wikitext.}
\label{tab:wiki}
\begin{tabular}{lcccccc}
\toprule
& \multicolumn{3}{c}{\textbf{Perplexity} $\bm{(\downarrow)}$} & \multicolumn{3}{c}{\textbf{16-bit sub-blocks per tile}} \\
\cmidrule(lr){2-4}
\cmidrule(lr){5-7}
\textbf{Model name} & 8-bit & \textbf{8/16-bit LAMP} & 32-bit & 0 & 1 & 2+ \\
\midrule
gemma-3-27b-pt & 8.6290 & \textbf{6.2100} & 5.9000 & 47.21\% & 19.86\% & 32.93\% \\
gemma-3-12b-pt & 15.940 & \textbf{8.6370} & 7.3790 & 52.86\% & 19.10\% & 28.04\% \\
\midrule
Qwen3-32B & 12.170 & \textbf{9.4290} & 9.3610 & 76.50\% & 11.55\% & 11.95\% \\
Qwen3-8B & 17.480 & \textbf{12.510} & 12.300 & 76.44\% & 11.01\% & 12.55\% \\
\midrule
Qwen3-30B-A3B & 18.780 & \textbf{11.260} & 10.890 & 81.82\% & 9.02\% & 9.16\% \\
\bottomrule
\end{tabular}
\end{table}

\begin{table}[htbp]
\centering
\caption{LampAttention with $\delta = 2^{-8}$ and $\tau = 2^{-1}$ on 0-shot ARC-Challenge.}
\label{tab:arc0}
\begin{tabular}{lcccccc}
\toprule
& \multicolumn{3}{c}{\textbf{Accuracy} $\bm{(\uparrow)}$} & \multicolumn{3}{c}{\textbf{16-bit sub-blocks per tile}} \\
\cmidrule(lr){2-4}
\cmidrule(lr){5-7}
\textbf{Model name} & 8-bit & \textbf{8/16-bit LAMP} & 32-bit & 0 & 1 & 2+ \\
\midrule
gemma-3-27b-pt & 0.5977 & \textbf{0.6484} & 0.6562 & 48.47\% & 30.41\% & 21.12\% \\
gemma-3-12b-pt & 0.4570 & \textbf{0.6191} & 0.6387 & 48.71\% & 32.21\% & 19.08\% \\
\midrule
Qwen3-32B & 0.6191 & \textbf{0.6016} & 0.6094 & 48.54\% & 31.22\% & 20.24\% \\
Qwen3-8B & 0.4766 & \textbf{0.5352} & 0.5605 & 48.24\% & 28.71\% & 23.05\% \\
\midrule
Qwen3-30B-A3B & 0.4512 & \textbf{0.5176} & 0.5762 & 48.28\% & 28.94\% & 22.78\% \\
\bottomrule
\end{tabular}
\end{table}

\begin{figure}[tbhp]
    \centering
    \includegraphics[width=\linewidth]{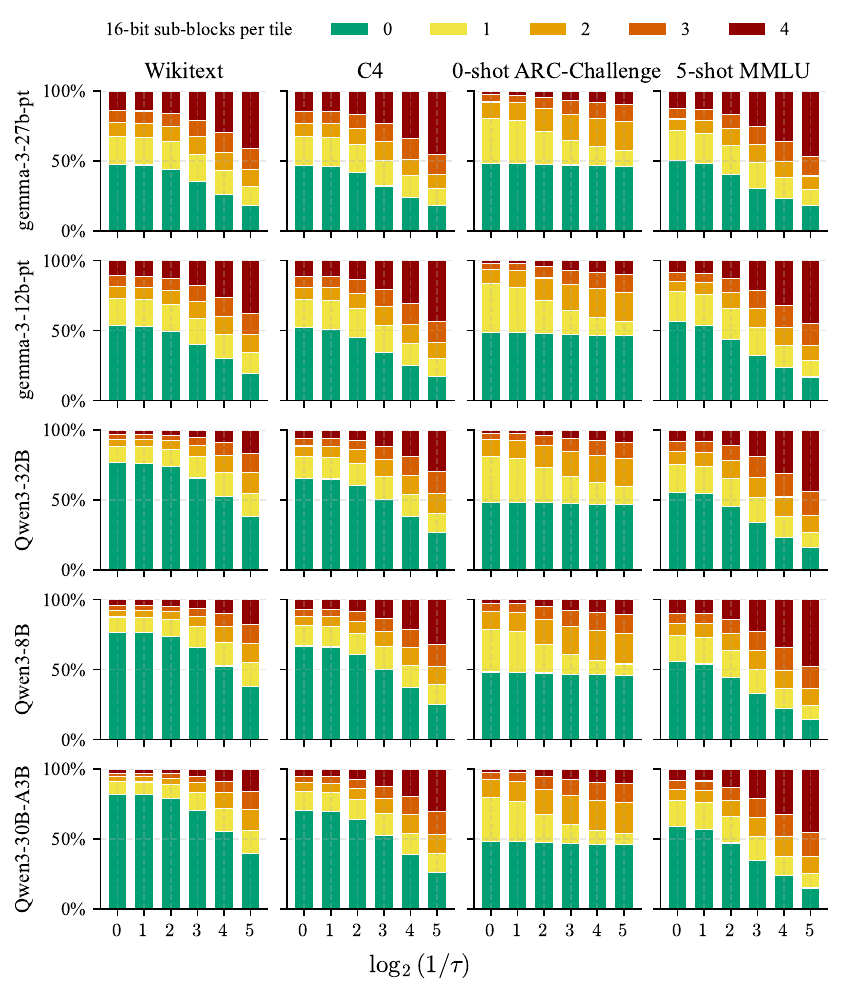}
    \caption{LampAttention with $\delta = 2^{-8}$ and $\tau \in \{ 2^{-t}~:~t = 0, \ldots, 5 \}$.}
    \label{fig:bars_5x4}
\end{figure}

\clearpage
\section{Impact of safety margin}
\label{appendix:margin}
Here, we compare the performance of LampAttention with two different values of the safety margin $\delta$ from the first stage of the algorithm.
As Figures~\ref{fig:pareto_ppl_5x2} and~\ref{fig:pareto_acc_5x2} show, larger $\delta$ results in more ``disruptive'' recomputations being initiated during the first stage (which $\tau = 1$ isolates).
At the same time, the two trade-off curves exhibit similar convergent behavior as the threshold $\tau$ tightens.

\begin{figure}[tbhp]
    \centering
    \includegraphics[width=\linewidth]{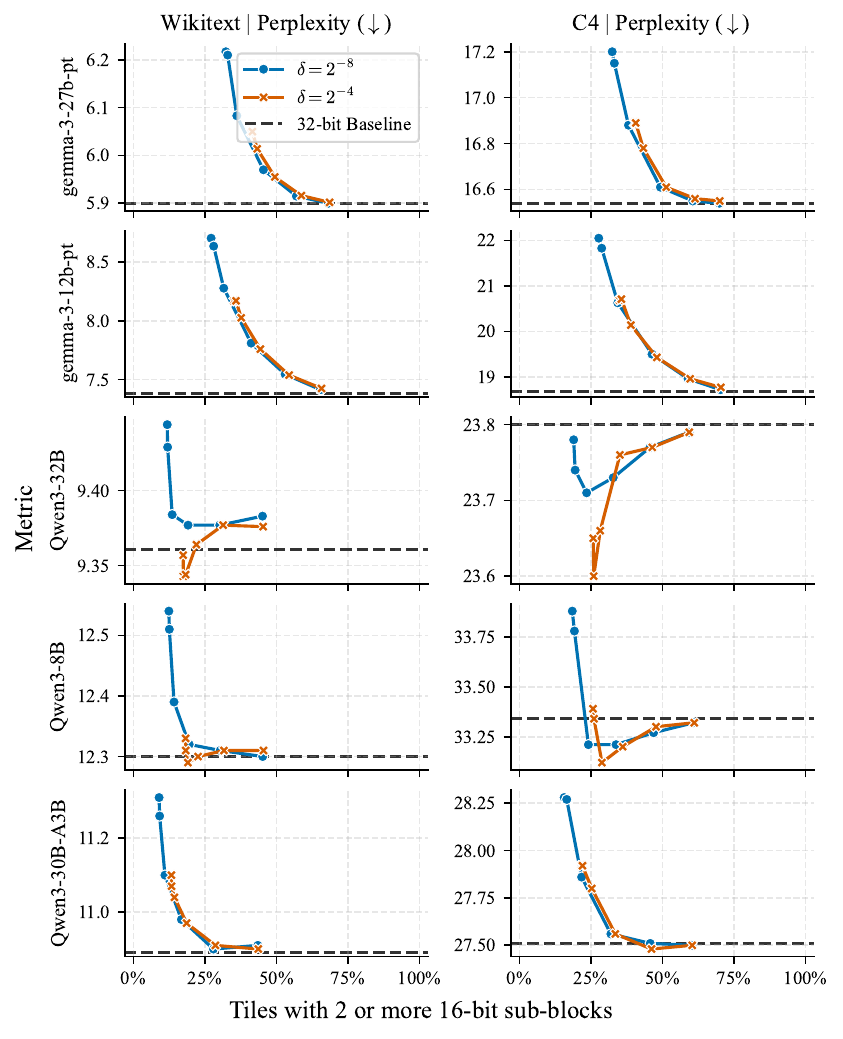}
    \caption{LampAttention with $\delta \in \{ 2^{-8}, 2^{-4} \}$ and $\tau \in \{ 2^{-t}~:~t = 0, \ldots, 5 \}$.}
    \label{fig:pareto_ppl_5x2}
\end{figure}

\begin{figure}[tbhp]
    \centering
    \includegraphics[width=\linewidth]{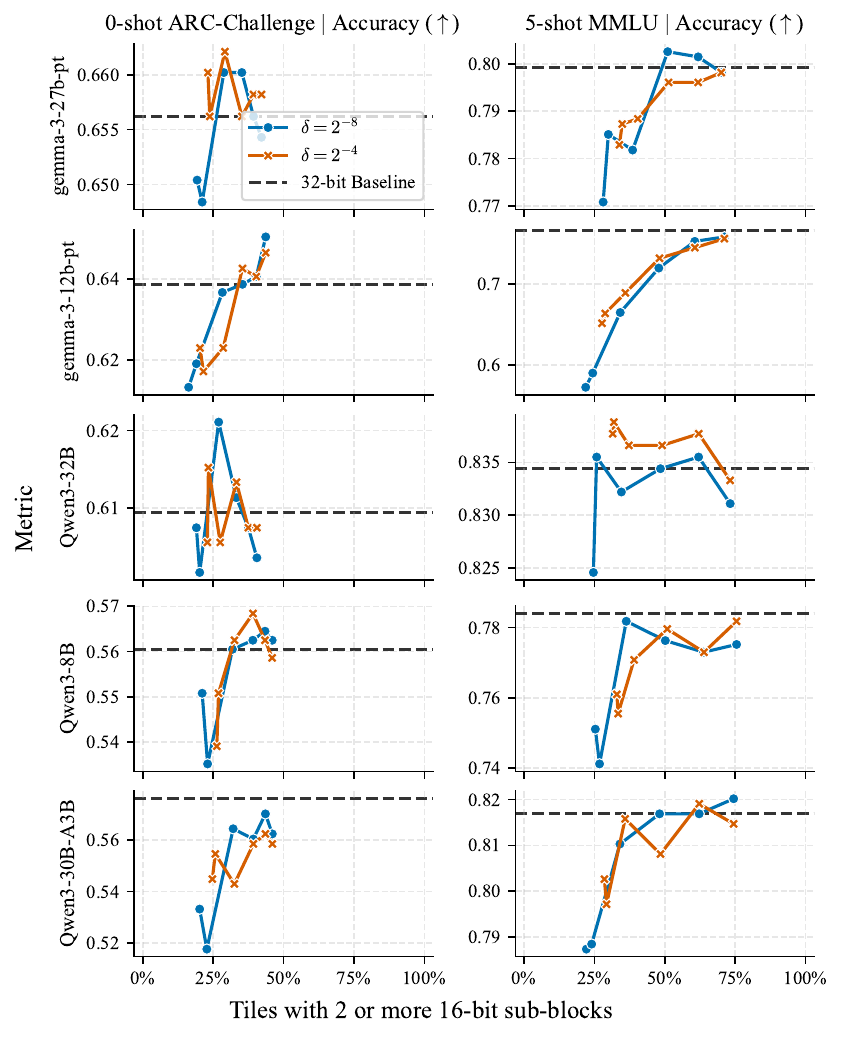}
    \caption{LampAttention with $\delta \in \{ 2^{-8}, 2^{-4} \}$ and $\tau \in \{ 2^{-t}~:~t = 0, \ldots, 5 \}$.}
    \label{fig:pareto_acc_5x2}
\end{figure}

\clearpage
\section{Impact of first-stage recomputation}
\label{appendix:stage}
While setting $\tau = 1$ effectively disables the second recomputation stage in Algorithm~\ref{alg:lampattn}, we can choose to disable the first stage instead (even though this violates the theoretical assumptions underlying the derivation of the block-LAMP problem \eqref{eq:lamp_block}).
Figures~\ref{fig:pareto_ppl_5x2_nomax} and~\ref{fig:pareto_acc_5x2_nomax} show the corresponding trade-off curves.
Taken together with the results in Appendix~\ref{appendix:margin}, they suggest that the specific mechanism triggering a recomputation is secondary; whether driven by the first-stage running maximum or the second-stage threshold, both phases identify fundamentally ``important'' sub-blocks.

\begin{figure}[tbhp]
    \centering
    \includegraphics[width=\linewidth]{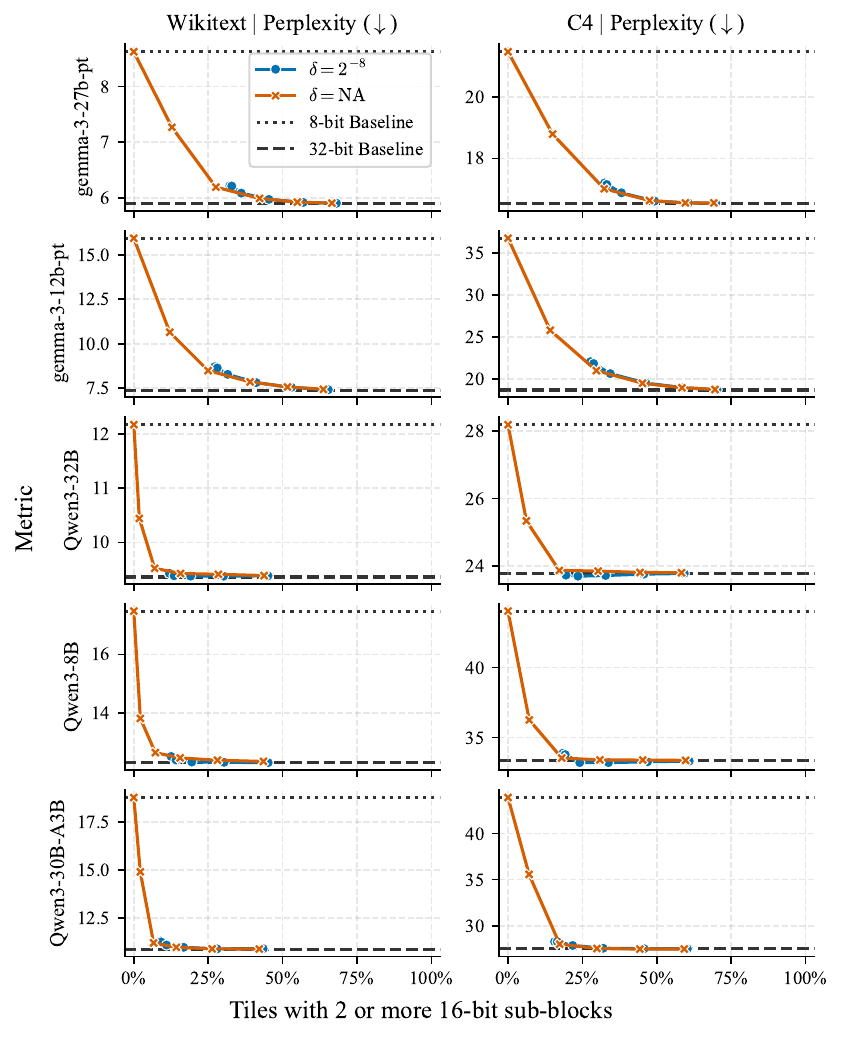}
    \caption{LampAttention with $\delta = 2^{-8}$ or disabled first stage, and $\tau \in \{ 2^{-t}~:~t = 0, \ldots, 5 \}$.}
    \label{fig:pareto_ppl_5x2_nomax}
\end{figure}

\begin{figure}[tbhp]
    \centering
    \includegraphics[width=\linewidth]{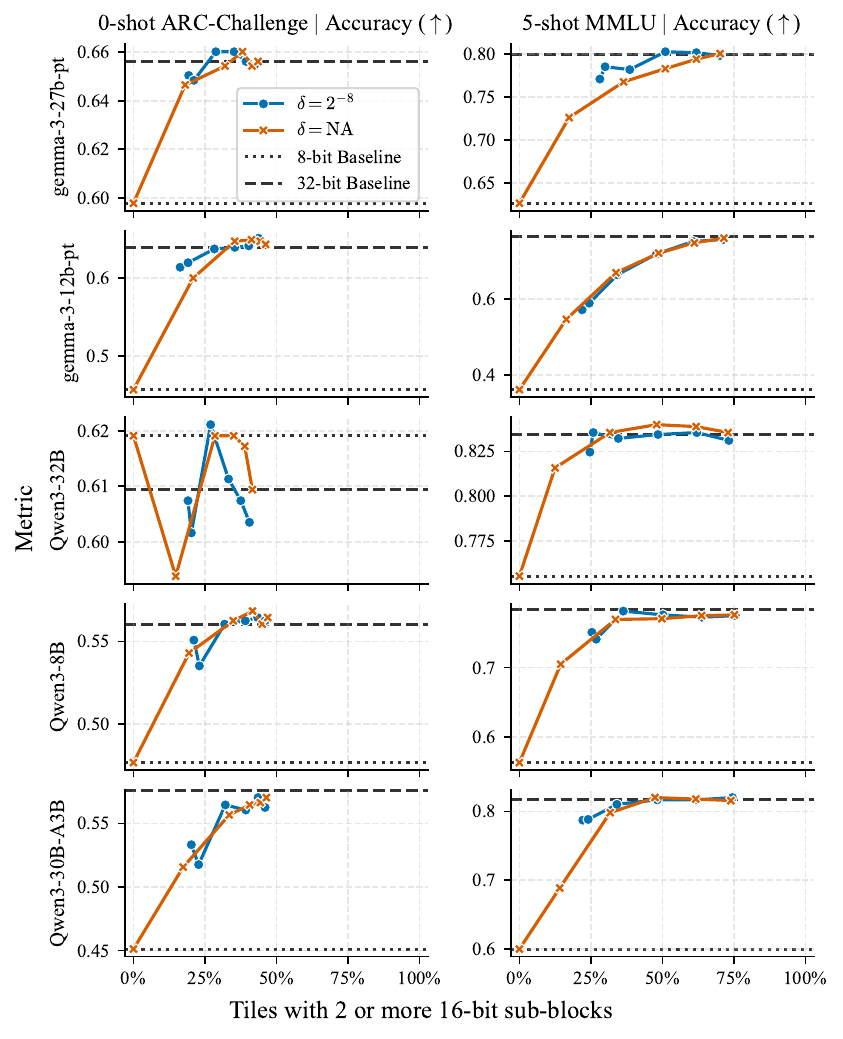}
    \caption{LampAttention with $\delta = 2^{-8}$ or disabled first stage, and $\tau \in \{ 2^{-t}~:~t = 0, \ldots, 5 \}$.}
    \label{fig:pareto_acc_5x2_nomax}
\end{figure}

\clearpage
\section{Impact of in-context learning}
\label{appendix:fewshot}
We assess how task familiarity affects the performance of LampAttention by comparing 0-shot and 25-shot evaluations on ARC-Challenge.
Tables~\ref{tab:arc25} and~\ref{tab:arc0_copy} (a reproduction of Table~\ref{tab:arc0} for convenience) demonstrate a dual benefit: the few-shot setting naturally increases the accuracy while simultaneously expanding the share of tiles processed entirely in 8-bit.
Interestingly, Figure~\ref{fig:bars_5x2_fewshot} reveals that the distribution of recomputed sub-blocks per tile evolves distinctively between the two task instances.

\begin{table}[htbp]
\centering
\caption{LampAttention with $\delta = 2^{-8}$ and $\tau = 2^{-1}$ on 25-shot ARC-Challenge.}
\label{tab:arc25}
\begin{tabular}{lcccccc}
\toprule
& \multicolumn{3}{c}{\textbf{Accuracy} $\bm{(\uparrow)}$} & \multicolumn{3}{c}{\textbf{16-bit sub-blocks per tile}} \\
\cmidrule(lr){2-4}
\cmidrule(lr){5-7}
\textbf{Model name} & 8-bit & \textbf{8/16-bit LAMP} & 32-bit & 0 & 1 & 2+ \\
\midrule
gemma-3-27b-pt & 0.6504 & \textbf{0.6953} & 0.7070 & 56.64\% & 19.37\% & 23.99\% \\
gemma-3-12b-pt & 0.4609 & \textbf{0.6289} & 0.6777 & 62.56\% & 19.31\% & 18.13\% \\
\midrule
Qwen3-32B & 0.6758 & \textbf{0.7305} & 0.7324 & 55.23\% & 19.55\% & 25.22\% \\
Qwen3-8B & 0.5488 & \textbf{0.6660} & 0.6660 & 54.10\% & 18.53\% & 27.37\% \\
\midrule
Qwen3-30B-A3B & 0.5879 & \textbf{0.7031} & 0.6953 & 57.72\% & 19.01\% & 23.27\% \\
\bottomrule
\end{tabular}
\end{table}

\begin{table}[htbp]
\centering
\caption{LampAttention with $\delta = 2^{-8}$ and $\tau = 2^{-1}$ on 0-shot ARC-Challenge.}
\label{tab:arc0_copy}
\begin{tabular}{lcccccc}
\toprule
& \multicolumn{3}{c}{\textbf{Accuracy} $\bm{(\uparrow)}$} & \multicolumn{3}{c}{\textbf{16-bit sub-blocks per tile}} \\
\cmidrule(lr){2-4}
\cmidrule(lr){5-7}
\textbf{Model name} & 8-bit & \textbf{8/16-bit LAMP} & 32-bit & 0 & 1 & 2+ \\
\midrule
gemma-3-27b-pt & 0.5977 & \textbf{0.6484} & 0.6562 & 48.47\% & 30.41\% & 21.12\% \\
gemma-3-12b-pt & 0.4570 & \textbf{0.6191} & 0.6387 & 48.71\% & 32.21\% & 19.08\% \\
\midrule
Qwen3-32B & 0.6191 & \textbf{0.6016} & 0.6094 & 48.54\% & 31.22\% & 20.24\% \\
Qwen3-8B & 0.4766 & \textbf{0.5352} & 0.5605 & 48.24\% & 28.71\% & 23.05\% \\
\midrule
Qwen3-30B-A3B & 0.4512 & \textbf{0.5176} & 0.5762 & 48.28\% & 28.94\% & 22.78\% \\
\bottomrule
\end{tabular}
\end{table}

\begin{figure}[tbhp]
    \centering
    \includegraphics[width=\linewidth]{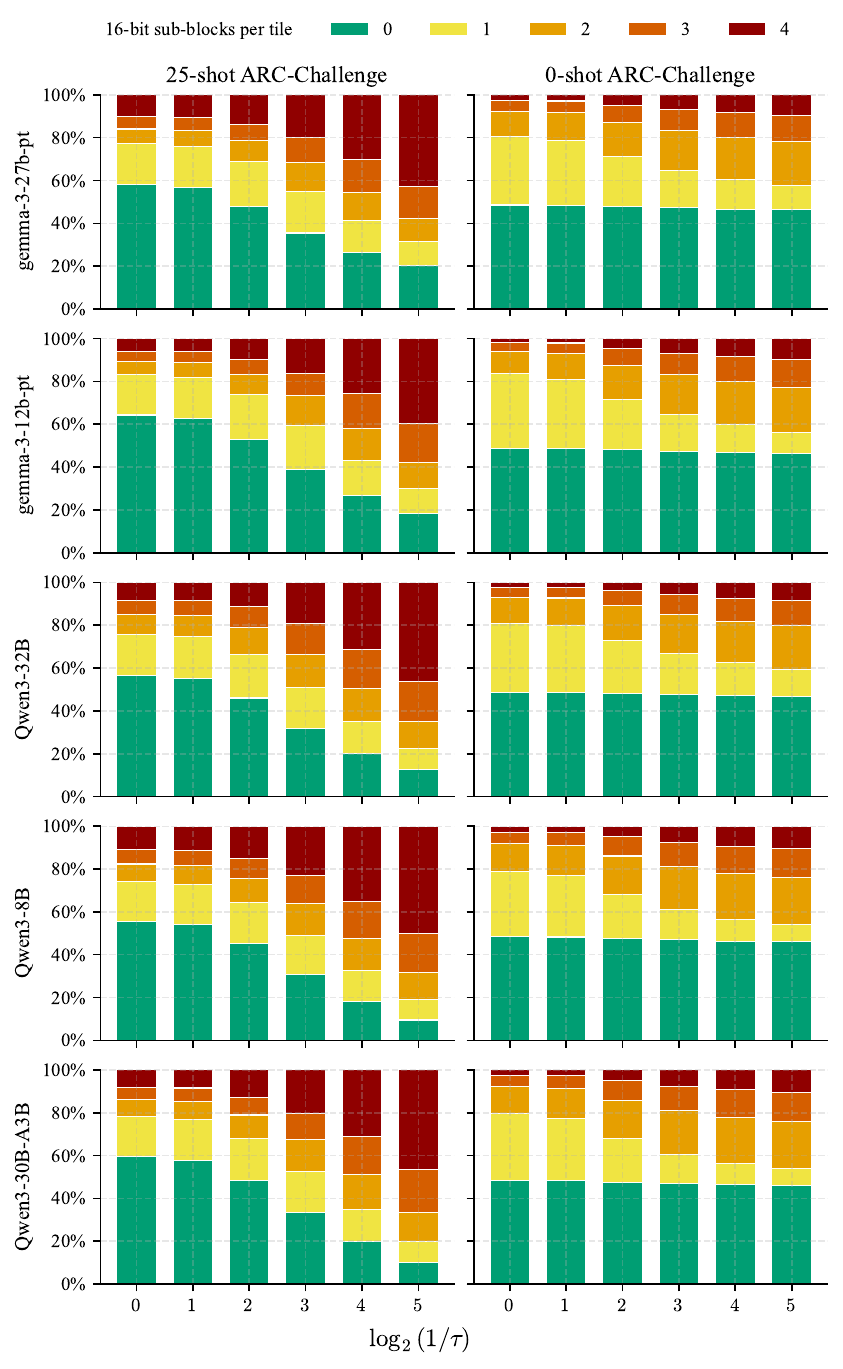}
    \caption{LampAttention with $\delta = 2^{-8}$ and $\tau \in \{ 2^{-t}~:~t = 0, \ldots, 5 \}$.}
    \label{fig:bars_5x2_fewshot}
\end{figure}

\end{document}